%% file: arxiv.tex
\documentclass[11pt]{article}

\usepackage[preprint]{latex/acl}

\usepackage{times}
\usepackage{latexsym}

\usepackage[T1]{fontenc}
\usepackage[utf8]{inputenc}

\usepackage{microtype}

\usepackage{inconsolata}

\usepackage{graphicx}
\usepackage{enumitem}

\input{latex/preamble}
\title{
On-the-Fly VLA Adaptation via Test-Time Reinforcement Learning}

\author{
 \textbf{Changyu Liu\textsuperscript{1\thanks{Equal contribution}}},
 \textbf{Yiyang Liu\textsuperscript{1\footnotemark[1]}},
 \textbf{Taowen Wang\textsuperscript{2}},
 \textbf{Qiao Zhuang\textsuperscript{1}},
\\
 \textbf{James Chenhao Liang\textsuperscript{3}},
 \textbf{Wenhao Yang\textsuperscript{4}},
 \textbf{Renjing Xu\textsuperscript{2}},
 \textbf{Qifan Wang\textsuperscript{5}},
\\
 \textbf{Dongfang Liu\textsuperscript{6\footnotemark[2]}},
 \textbf{Cheng Han\textsuperscript{1\thanks{Corresponding authors}}}
\\
\\
 \textsuperscript{1}University of Missouri--Kansas City,
 \textsuperscript{2}Hong Kong University of Science and Technology (Guangzhou),\\
 \textsuperscript{3}U. S. Naval Research Laboratory,
 \textsuperscript{4}Lamar University,
 \textsuperscript{5}Meta AI,
 \textsuperscript{6}Rochester Institute of Technology
}

\begin{document}
\maketitle
\input{sec/0_abstract}
\input{sec/1_intro}
\input{sec/2_rl}

\input{sec/3_method}
\input{sec/4_exp}

\input{sec/6_conclusion}

% \newpage
%\input{sec/limitation}
\input{sec/acknowledge}

\bibliography{latex/custom}

\input{sec/X_suppl}

\end{document}

%% file: latex/preamble.tex
\usepackage{xcolor}

\definecolor{myred}{RGB}{255 0 0} 
\definecolor{mygreen2}{RGB}{0 176 80} 
\definecolor{mygreen3}{RGB}{80,99,43} 
\definecolor{mygray}{gray}{.9}

\makeatletter
\newcommand{\thickhline}{%
    \noalign {\ifnum 0=`}\fi \hrule height 1pt
    \futurelet \reserved@a \@xhline
}

\usepackage{adjustbox}
\usepackage{pifont}

\usepackage{multirow}
\usepackage{xcolor}
\usepackage{colortbl}
\usepackage{booktabs}
\usepackage[ruled,vlined]{algorithm2e}

\usepackage{amsthm}

\newtheorem{lemma}{Lemma}
\newtheorem{proposition}{Proposition}
\newtheorem{corollary}{Corollary}
\definecolor{mygray}{gray}{.9}

\newcolumntype{Y}{>{\raggedleft\arraybackslash}X}

\usepackage{times}    % Integrate Times for here
\usepackage{xspace}
\usepackage{graphicx}
\usepackage{amsmath}
\usepackage{amssymb}
\usepackage{scalerel}
\usepackage{tabularx}

\definecolor{ForestGreen}{RGB}{27,94,32}

\DeclareRobustCommand\onedot{\futurelet\@let@token\@onedot}
\def\@onedot{\ifx\@let@token.\else.\null\fi\xspace}
\def\eg{\emph{e.g}\onedot} 
\def\ie{\emph{i.e}\onedot} 

%% file: sec/0_abstract.tex
\begin{abstract}
Vision-Language-Action (VLA) models have recently emerged as a powerful paradigm for general-purpose robot learning, enabling agents to map visual observations and natural-language instructions into executable robotic actions. 
Though popular, they are primarily trained via supervised fine-tuning or training-time 
reinforcement learning, requiring explicit fine-tuning phases, human interventions, or controlled data collection.
Consequently, existing methods remain unsuitable for challenging simulated- or physical-world deployments, where robots must respond autonomously and flexibly to evolving environments.
To address this limitation, we introduce a Test-Time Reinforcement Learning for VLAs (TT-VLA), a framework that enables on-the-fly policy adaptation during inference. TT-VLA formulates a dense reward mechanism that leverages step-by-step task-progress signals to refine action policies during test time while preserving the SFT/RL-trained priors, making it an effective supplement to current VLA models. 
Empirical results show that our approach enhances overall adaptability, stability, and task success in dynamic, previously unseen scenarios under simulated and real-world settings. We believe TT-VLA offers a principled step toward self-improving, deployment-ready VLAs.
\end{abstract}

%% file: sec/1_intro.tex
\vspace{-2mm}
\section{Introduction}
\label{sec:intro}

The ability to adapt to changing conditions is a fundamental requirement for intelligent agents operating in the real world. However, agents trained in fixed, well-defined environments frequently fail when confronted with the inherent real-world dynamic variability~\cite{dulac2021challenges, tambe1995intelligent, kormushev2013reinforcement}. This gap between static training regimes and dynamic deployment environments remains a central challenge in robotics and embodied intelligence.

\begin{figure}[t]
    \centering
    \includegraphics[width=0.46\textwidth]{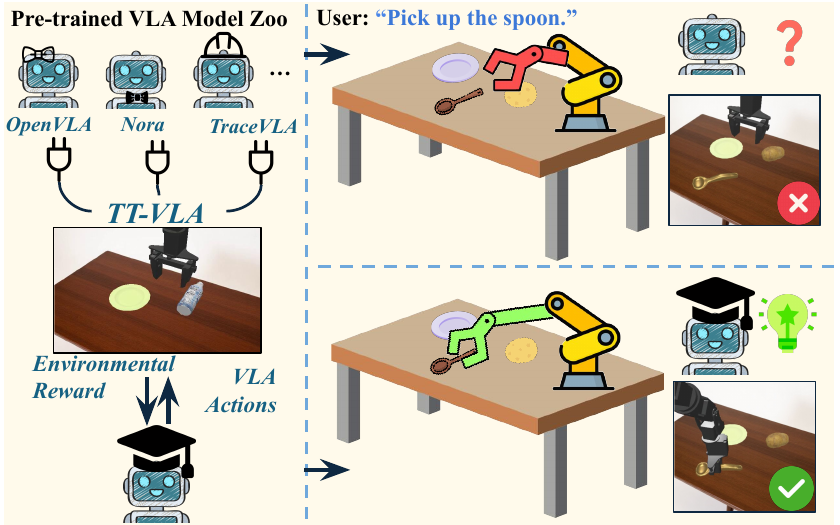}
    \vspace{-1mm}
    \caption{\textbf{TT-VLA supplements to SFT/RL-trained VLAs} by continuously adapting policies to environment-derived rewards at test time, improving robustness to distributional shifts without retraining.
    }
    \vspace{-7mm}
    \label{fig:overview}
\end{figure}

Current Vision-Language-Action (VLA) models stand at the heart of this challenge. These VLAs integrate perception, language understanding, and control to translate visual observations and natural language instructions into executable actions, representing a significant step and performance gains 
toward general-purpose embodied intelligence~\cite{kim24openvla, Brohan2023rt1, Brianna2023RT2, wang2025vla, kwok2025robomonkey}. Yet, 
most VLAs remain bound to fixed policies, which are generally trained either through supervised fine-tuning (SFT) or through training-time reinforcement learning (RL) on curated datasets, explicit fine-tuning phases, and controlled environments.

In practice, these rigid policies limit VLAs' capacity 
for challenging simulated-/physical-world applications, where \textbf{\textit{robots must adapt at test time as conditions evolve or distribution shifts}}. 
In a broader perspective, current research in language or vision 
understanding has begun to explore 
test-time training (TTT)~\cite{zuo2025ttrl, li2025system, karmanov2024efficient, hu2024bafta, ma2023swapprompt} to update models directly on incoming test streams, often without ground-truth labels, underscoring a promising direction toward flexible model adjustments. These advances have emerged in other domains, leaving VLA test-time adaptation largely underexplored. We further find that current TTT cannot be directly applied to VLAs, as the multimodal nature brings substantial distributional shifts and is evolving (see \S\ref{subsec:discussion}).

In light of this view, we propose a \underline{T}est-\underline{T}ime Reinforcement Learning framework for \underline{VLA} (TT-VLA) that performs online inference-time policy fine-tuning efficiently without retraining cycles, preserving SFT/training-time RL priors while closing the loop with dense inference-time reward signals. 
Here, TTT provides accessibility for test-time adaptation, while RL complements it by addressing substantial variations in environmental conditions and underlying distributions.

Different from traditional session-based RL (\ie, SimpleVLA-RL), we creatively derive dense shaping signals from task-agnostic proxies to timely and effectively utilize the limited test-time information and optimize the VLA policy. These shaping signals operate independently across time frames, enabling stable, versatile, and continuous adjustments.
Our design also bridges the prevailing offline-RL/VLA pipeline and the demands of real-world robotics, enabling continuous self-improvement under dynamic, previously unseen conditions. Extensive experiments conducted in both simulated and real-world robotic environments demonstrate that our method can naturally boost the performance of existing SFT-/RL-based approaches.

%% file: sec/2_rl.tex
\section{Related Work}
\label{sec:rl}

\subsection{VLA Generalization \& Adaptation}\label{subsec:lr:ga}
Current VLAs~\cite{Brohan2023rt1, MeesHB2022What, Pong2020Skew, kwok2025robomonkey} are predominantly optimized via supervised fine-tuning on large, curated triplets~\cite{Brianna2023RT2, kim24openvla, plaat2024reasoning, jiang2025alphadrive, kim2025robot, liu2024robomamba}, which yields strong performance in static, well-structured settings but leads to brittle behavior and limited robustness under distributional shifts due to the lack of adaptive learning mechanisms~\cite{shenfeld2025rl, chen2025sft, xu2024rldg, chen2025conrft, wang2024exploring}.

To address these limitations, recent studies have explored integrating VLAs with reinforcement learning. RL allows policies to actively interact with pre-collected environments or demonstrations, thereby enabling continual adaptation and performance improvement beyond static supervision~\cite{huang2025co, zhang2025reinbot, mark2024policy, chen2025tgrpo, ye2025vla, li2025simplevla, lu2025vla, jiang2025irl, chen2025rlrc, wu2021online, guo2025comprehensive}. By incorporating interaction-driven feedback, RL-augmented VLAs refine their behaviors to task-specific objectives~\cite{huang2025co, zhang2025reinbot, mark2024policy} and environmental variations~\cite{chen2025tgrpo, ye2025vla, li2025simplevla, lu2025vla, jiang2025irl, chen2025rlrc, wu2021online, li2025vla, guo2025comprehensive}, demonstrating improved sample efficiency and generalization across unseen scenarios.
Despite these advances, existing RL approaches still adhere to a training-deployment separation paradigm,
without mechanism for inference-time adaptation once the model is deployed. This gap leaves VLAs vulnerable to unanticipated dynamics
during testing, where
retraining is impractical due to time or data constraints.

\subsection{Test-Time Training}\label{appendix:subsec:TTL}
Test-time training (TTT) is originally proposed for adapting pre-trained models to a target domain using only unlabeled test data~\cite{sun2020test,hu2025test,yoonc,xiao2025dynaprompt,zhu2024efficient,zuo2025ttrl}. Unlike SFT~\cite{jia2022visual,mosbach2020stability,han20232,wang2024exploring,liu2025re,liu2024alora,hu2022lora,zeng2024visual, wang2024mmpt} or traditional RL~\cite{sutton1999reinforcement,guo2025deepseek,huang2025co,zhang2025reinbot, mark2024policy}, 
TTT leverages self-supervised objectives (\eg, entropy minimization) to calibrate models during inference, thereby enabling effective domain adaptation in the absence of both human-curated labels and external feedback.

Intuitively, TTT can be naturally extended to VLAs to enable efficient adaptation during inference. However, unlike relatively intuitive single-domain tasks (\eg, vision, language), where pre-trained models exhibit high generalization capability~\cite{brown2020language,wei2021finetuned,goyal2023finetune} and inter-task discrepancies are relatively minor~\cite{hu2025test,liu2021ttt++,zhao2023pitfalls,han2025test},
robotic tasks often entail substantial distributional shifts and evolving conditions across both visual and linguistic modalities~\cite{Pong2020Skew,wang2024exploring,kim24openvla,liu2024robomamba}. Consequently,
fixed, protocol-driven self-supervised objectives become inadequate (see \S\ref{subsec:discussion}) and overly generic.

Recent work has explored reinforcement-based alternatives to the purely self-supervised test-time training objectives. In particular, EVOLVE-VLA~\cite{bai2025evolve} leverages task progress as a reward signal to optimize VLA policies during deployment. However, its reliance on GRPO-style optimization incurs substantial computational overhead, limiting its applicability to real-time robotic settings with strict latency constraints. We provide a more detailed discussion regarding the use of GRPO for TTT in VLAs in Appendix~\S\ref{appendix:grpo}.

To address these limitations, we propose an RL-driven test-time learning framework that enables efficient online adaptation (see \S\ref{sec:method}).

%% file: sec/3_method.tex
\section{Method}
\label{sec:method}
In this section, we introduce
\underline{T}est-\underline{T}ime Reinforcement Learning framework for \underline{VLA} (TT-VLA), a novel 
test-time training approach designed to enhance VLA performance through on-the-fly policy adaptation.
In \S\ref{subsec:prelim}, we first 
provide background on Proximal Policy Optimization (PPO) and VLAs. We then present TT-VLA in \S\ref{subsec:method}.
\S\ref{subsec:theory} further provides TT-VLA's theoretical analysis and insights. The overall framework is shown in Fig.~\ref{fig:pipeline}. 

\subsection{Preliminaries}
\label{subsec:prelim} 
\noindent \textbf{Problem Formulation.}
We model robotic manipulation as a Partially Observable Markov Decision Process (POMDP)~\cite{kaelbling1998planning}, defined as:
\begin{equation}
\vspace{-0.2em}
\scalebox{0.80}{$
\mathcal{M} = (\mathcal{S}, \mathcal{A}, \mathcal{O}, \mathcal{L}),
$}\vspace{-0.2em}
\end{equation}
where $\mathcal{S}$ denotes the state space of the robot and environment, $\mathcal{A}$ is the action space, $\mathcal{O}$ represents the multimodal observation space (\eg, RGB and proprioception), and $\mathcal{L}$ denotes the space of natural-language task instructions. 
At the start of a task episode, the VLA policy $\pi_\theta$ receives an instruction $l \in \mathcal{L}$ and an initial observation $o_0$. 
The goal of the VLA policy $\pi_\theta$ is to generate a sequence of actions:
\begin{equation}\vspace{-0.2em}
\scalebox{0.80}{$
a_{0:T-1} \sim \pi_\theta(a_t \mid o_{t-H+1:t}, l),
$}\vspace{-0.2em}
\end{equation}
where $o_t$ and $a_t$ denote the observation and action at time $t$, $T$ is the episode length, and $H$ is the number of past observations used as policy input.

The above formulation characterizes the standard VLA decision process. However, it inherently assumes a fixed, pre-trained policy. Real-world deployments, on the other hand, demand adaptability to dynamic environments.
Therefore, under the test-time adaptation, our goal should now switch to adjusting the pretrained policy $\pi_\theta$ online during deployment flexibly \textit{without} any access to training data, environment resets, or human intervention. 

\noindent\textbf{Proximal Policy Optimization (PPO)}. 
PPO is an actor–critic policy-gradient method that uses a clipped surrogate objective to constrain policy updates, ensuring stable optimization by keeping the updated policy within a trust region of the previous policy. Formally, the PPO objective is defined as: 
\begin{equation}
\label{eq:PPO}
\vspace{-0.2em}
\scalebox{0.75}{$
L^{\text{PPO}}(\theta) =
\mathbb{E}_t \left[
    L^{\text{CLIP}}_t(\theta)
    - c_1 L^{\text{Value}}_t(\theta)
    + c_2 L^{\text{entropy}}_t(\theta)
\right],
$}
\vspace{-0.2em}
\end{equation}
where $L_t^{\text{CLIP}}(\theta)$ represents the clipped policy loss, 
$L_t^{\text{Value}}(\theta)$ denotes the value function loss, $L_t^{\text{Entropy}}(\theta)$ is the entropy regularization term, and $c_1$ and $c_2$ are weighting coefficients. 
The clipped policy objective is defined as:
\begin{equation}
\label{eq:PPO_clip}
\vspace{-0.2em}
\scalebox{0.72}{$
L^{\text{CLIP}}_t(\theta) =
\mathbb{E}_t \left[
    \min\!\left(
        r_t(\theta) \hat{A}_t,\,
        \text{clip}\!\big(r_t(\theta), 1 - \epsilon, 1 + \epsilon\big)\hat{A}_t
    \right)
\right],
$}\vspace{-0.2em}
\end{equation}
where
$r_t(\theta) = \frac{\pi_\theta(a_t \mid s_t)}{\pi_{\theta_{\text{old}}}(a_t \mid s_t)}$
is the ratio between the new and old policy, 
$\epsilon$ controls the clipping range, $\hat{A}_t$ denotes the advantage estimate, and $\text{clip}(\cdot)$ denotes the clipping operation. 
PPO typically employs Generalized Advantage Estimation (GAE) to estimate $\hat{A}_t$, 
given by:
\begin{equation}
\label{eq:GAE}
\vspace{-0.2em}
\scalebox{0.78}{$
\hat{A}_t = \sum_{l=0}^{\infty} (\gamma \lambda)^l \, \delta_{t+l},
$}
\vspace{-0.3em}
\end{equation}
where
$\delta_t = r_t + \gamma V(s_{t+1}) - V(s_t)$
is temporal-difference (TD) residual, $\gamma$ is the discount factor, $\lambda$ is the smoothing parameter, and $V(s_t)$ is 
the estimated expected return from state $s_t$.

\subsection{TT-VLA}
\label{subsec:method} 
In PPO, both the policy $\pi_\theta$ and value function $V_\theta$ are trained jointly. However, in VLA test-time adaptation,
learning a reliable value function is generally infeasible due to two reasons: \textit{\textbf{\ding{172}}} \textit{\textbf{Limited samples}}: Test-time adaptation relies on extremely limited interaction data, a single episode, which is insufficient for accurate return estimation. \textit{\textbf{\ding{173}}}~\textit{\textbf{Strict time constraints}}: Test-time updates for VLA models must be performed online under strict latency constraints. To overcome these limitations, we develop a value-free PPO, which enables policy adaptation \textit{without} an explicit value function. 

\begin{figure*}[t]
    \centering
    \includegraphics[width=0.98\textwidth]{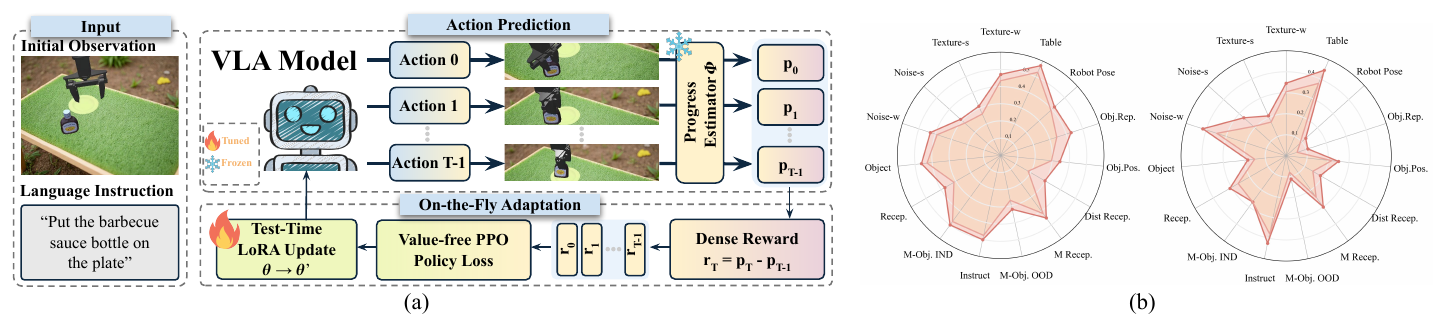}
    \vspace{-6pt}
    \caption{
    \textbf{Overview of TT-VLA.}
    \textbf{(a) Overall pipeline.} 
    In TT-VLA, a pretrained VLA policy receives an observation and instruction, executes actions in the environment, and receives dense, progress-based rewards computed by a progress estimator. These rewards are used to update the policy online via a value-free PPO objective, enabling continuous within-episode policy adaptation at test time (see \S\ref{subsec:method}).
    \textbf{(b) Effectiveness.} TT-VLA consistently improves the performance of diverse VLA backbones across unseen tasks, demonstrating robust generalization and adaptability under evolving conditions or distributional shifts (see \S\ref{subsec:sim-result}-\ref{subsec:real-world}).
    }
    \label{fig:pipeline}
    \vspace{-10pt}
\end{figure*}

\noindent\textbf{Dense Progress-Based Reward.}
Most existing RL-based VLAs~\cite{li2025simplevla, liu2025can} rely on sparse terminal rewards, typically indicating binary task success or failure at the end of an episode. While such rewards are effective during offline training, where episodes can be replayed or reset, they are impractical for test-time adaptation: 
the policy updates are delayed until task completion, preventing any mid-episode correction or online adjustment.
Consequently, a robot operating with sparse rewards cannot refine its behavior during inference, leading to fragile and non-adaptive performance in dynamic environments.

Let $p_t \in [0, 1]$ denote task progress at time step $t$. Intuitively, $p_t$ should increase when actions advance task completion 
and decrease when actions undo or reverse previously achieved progress. In a POMDP setting, we estimate progress directly from observations and language instructions as:
\begin{equation}
\label{eq:progress_estimation}
\vspace{-0.2em}
\scalebox{0.80}{$
p_t = \Phi(o_{0:t+1}, l),
$}
\vspace{-0.2em}
\end{equation}
where $\Phi$ denotes a task progress predictor conditioned on the observation history and instruction $l$.  The per-step dense reward is then defined as the temporal difference in progress: 
\begin{equation}
\label{eq:reward}
\vspace{-0.2em}
\scalebox{0.80}{$
r_t = p_t - p_{t-1}.
$}
\vspace{-0.2em}
\end{equation}
We instantiate $\Phi$ using the Vision-Language-Action-Critic model (VLAC)~\cite{zhai2025vision}, a pretrained multimodal model that serves as a scalar regressor for task progress estimation.

This progress-based reward exhibits three desirable properties. First, it requires no external supervision during inference, allowing fully autonomous adaptation at test time. 
Second, it provides dense, step-wise feedback that facilitates continuous mid-episodic policy adaptation. Third, it encourages monotonic progress toward task completion while discouraging oscillatory or regressive behaviors. 

\noindent\textbf{Training Objective.}
Under the test-time VLA setting, learning an accurate value function $V_\theta$ within a single episode is impractical. 
We therefore adopt a value-free PPO variant that removes the value function learning and directly uses the per-step reward signal from Eq.~\ref{eq:reward} for policy adaptation. 

Starting from Eq.~\ref{eq:PPO}, 
we remove auxiliary losses by setting $c_1=0$ and $c_2=0$, discarding both value regression and entropy regularization term.
While the entropy term encourages exploration during training, test-time adaptation prioritizes rapid fitting of the current task rather than broad exploration. As a result, our objective focuses solely on stable policy refinement, while preserving the clipped surrogate loss. Eq.~\ref{eq:PPO} now turns into:
\begin{equation}
\label{eq:PPO_2}
\vspace{-0.2em}
\scalebox{0.80}{$
L(\theta) =
\mathbb{E}_t \left[
    L^{\text{CLIP}}_t(\theta)
\right].
$}
\vspace{-0.2em}
\end{equation}

To make the agent precisely capture the immediate value of the current action, we further redefine the advantage to depend only on the reward obtained from that action, rather than the exponentially weighted combination of TD residuals used in GAE (see Eq.~\ref{eq:GAE}). In other words, we focus on how each individual action contributes to instantaneous progress rather than estimating long-term returns.  To achieve this, we set $\lambda = 0$ and $\gamma = 0$, collapsing GAE into a one-step formulation:
\begin{equation}
\label{eq:new_advantage}
\vspace{-0.2em}
\scalebox{0.80}{$
\hat{A}_t = \delta_{t} = r_t.
$}
\vspace{-0.2em}
\end{equation}
Here, $\delta_{t} = r_t$ since we remove the value function, making the TD residual the immediate reward signal.
This simplification ensures that policy updates directly reflect the progress at each step, allowing the agent to adapt rapidly to changing conditions without relying on a value function. 

\noindent\textbf{Overall Pipeline.}
At the beginning of each episode, the pretrained VLA receives an initial observation $o_0$ and instruction $l$. The VLA policy $\pi_\theta$ generates the first action $a_0$. At each subsequent time step $t$, the VLA receives the latest observation $o_t$ and outputs an action $a_t$. After execution, the progress estimator $\Phi$ computes the task progress $p_t$, and the corresponding dense reward $r_t$ (see Eq.~\ref{eq:reward}). This reward is used to compute the policy loss via Eq.~\ref{eq:PPO_2} in a value-free manner, and the policy parameters $\theta$ are updated accordingly. 
The updated policy $\pi_{\theta'}$ is then used to generate subsequent actions, allowing continuous refinement throughout the episode. 
The pseudo code is shown in Appendix Algorithm~\ref{alg:ttvla}.

\subsection{Theoretical Analysis of TT-VLA}
\label{subsec:theory} 
In \S\ref{subsec:method}, we proposed TT-VLA using a progress-based dense reward and a value-free formulation for test-time adaptation. In this section, we provide a theoretical justification for these design choices. 

\begin{proposition}[Vanishing learning signal under progress-difference reward]
Let the per-step reward be defined as the progress difference
$r_t = p_t - p_{t-1}$ with $p_t \in [0,1]$.
Assume that the value function represents the remaining progress,
$V(s_t) = 1 - p_{t-1}$, and the discount factor is $\gamma = 1$.
Then the temporal-difference (TD) error vanishes for all $t$, and consequently
the GAE $\hat{A}_t$ is identically zero:
\begin{equation}
\vspace{-0.2em}
\scalebox{0.80}{$
\delta_t = 0, \quad \hat{A}_t = 0, \qquad \forall \lambda \in [0,1].
$}
\vspace{-0.2em}
\end{equation}
\end{proposition}

\begin{proof}
Substituting $r_t = p_t - p_{t-1}$, $V(s_t)=1-p_{t-1}$, and $\gamma=1$ into the TD residual yields
\begin{equation}
\vspace{-0.2em}
\scalebox{0.85}{$
\begin{aligned}
\delta_t 
&= r_t + \gamma V(s_{t+1}) - V(s_t) \\
&= (p_t - p_{t-1}) + (1 - p_t) - (1 - p_{t-1}) \\
&= 0.
\end{aligned}
$}
\vspace{-0.2em}
\end{equation}
By Eq.~\ref{eq:GAE}, GAE is a weighted sum of TD residuals, it follows that $\hat{A}_t = 0$. Therefore, the policy gradient receives no
learning signal.
\end{proof}

\begin{corollary}[Negative TD bias when $\gamma<1$]
Let $r_t = p_t - p_{t-1}$ with $p_t \in [0,1]$ and $V(s_t) = 1 - p_{t-1}$.
If $0<\gamma<1$, 
then TD residual $\delta_t<0$, introducing a systematic negative bias in advantage estimation.
\end{corollary}

\begin{proof}
Substituting $r_t = p_t - p_{t-1}$ and $V(s_t)=1-p_{t-1}$ into the TD residual, we get
\begin{equation}
\vspace{-0.2em}
\scalebox{0.85}{$
\begin{aligned}
\delta_t 
&= r_t + \gamma V(s_{t+1}) - V(s_t) \\
&= (p_t-p_{t-1}) + \gamma(1-p_t) - (1-p_{t-1})\\
&= (\gamma-1)(1-p_t).
\end{aligned}
$}
\vspace{-0.2em}
\end{equation}
Since $0<\gamma<1$ and $1-p_t>0$, this implies $\gamma-1<0$ and thus $\delta_t<0$.
\end{proof}

\begin{lemma}[One-step collapse of GAE]
\label{lemma:1}
Let the reward be defined as the progress difference
\begin{equation}
\vspace{-0.2em}
\scalebox{0.85}{$
r_t = p_t - p_{t-1},
$}
\vspace{-0.2em}
\end{equation}
and let the value function estimator be \(V(s)\).
Then, for GAE:
\begin{equation}
\vspace{-0.2em}
\scalebox{0.85}{$
\hat{A}_t = \sum_{l=0}^{\infty} (\gamma\lambda)^l \, \delta_{t+l}, 
$}
\vspace{-0.2em}
\end{equation}

\begin{equation}
\vspace{-0.2em}
\scalebox{0.85}{$
\delta_t 
= r_t + \gamma V(s_{t+1}) - V(s_t),
$}
\end{equation}
the following statements hold:\\
(a). If $\lambda = 0$, then $A_t = \delta_t$.\\
(b). If $\gamma=0$, then $A_t = \delta_t = r_t - V(s_t)$. In part-\\
icular, if $V(s) \equiv 0$, then $A_t = r_t$. 
\end{lemma}

The proof is provided in Appendix~\S\ref{appendix:proof_lemma}.

%% file: sec/4_exp.tex
\section{Experiment}
\label{sec:exp}
In this section, we present a comprehensive evaluation of our proposed method through a series of unseen robotic tasks. We detail the task setups, outline the implementation specifics, and compare our approach against baselines.
More experimental details are provided in Appendix \S\ref{appendix:tasks_detail}-\ref{appendix:moreevaluation}.

\begin{table*}[t]
\centering
\scriptsize
\setlength{\tabcolsep}{1.5pt}

\begin{tabularx}{\textwidth}{l||YYY|Y||YYYYY|Y}
\hline \thickhline
\rowcolor{mygray}
&\multicolumn{4}{c||}{\textit{Execution}} &\multicolumn{6}{c}{\textit{Vision}}\\
\rowcolor{mygray}
\multirow{-2}{*}{Model}
& \textbf{Obj. Pos.} & \textbf{Robot Pose} & \textbf{Obj. Rep.}& \textbf{Avg.}
& \textbf{Table} & \textbf{Texture-w} & \textbf{Noise-w} & \textbf{Texture-s} & \textbf{Noise-s} & \textbf{Avg.}\\
\hline\hline
Nora \tiny~\cite{hung2025nora}
& 23.75\% & 10.83\% & 7.50\%& 14.03\%
& 39.72\% & 32.50\% & 36.67\% & 19.58\% & 22.92\% & 29.92\%\\
Nora + TT-VLA
& 25.00\% & 12.50\% & 10.83\%&16.11\%
& 44.58\% & 34.58\% & 41.67\% & 20.83\% & 27.08\% &33.75\%\\
$\Delta$
& +1.25\% & +1.67\% & +3.33\%&2.08\%
& +6.66\% & +2.08\% & +5.00\% & +1.25\% & +4.16\% &+3.83\%\\
\rowcolor{mygray}
Relative Gain& \textcolor{ForestGreen}{$\uparrow$ 5.26\%} & \textcolor{ForestGreen}{$\uparrow$ 15.42\%}
& \textcolor{ForestGreen}{$\uparrow$ 44.40\%} & \textcolor{ForestGreen}{$\uparrow$ 14.85\%}
& \textcolor{ForestGreen}{$\uparrow$ 17.56\%}
& \textcolor{ForestGreen}{$\uparrow$ 6.40\%} & \textcolor{ForestGreen}{$\uparrow$ 13.64\%}
& \textcolor{ForestGreen}{$\uparrow$ 6.38\%} & \textcolor{ForestGreen}{$\uparrow$ 18.15\%} &\textcolor{ForestGreen}{$\uparrow$ 12.80\%}\\
\hline
OpenVLA \tiny~\cite{kim24openvla}
& 31.67\% & 41.25\% & 36.25\%&36.39\%
& 54.85\% & 45.42\% & 40.83\% & 28.33\% & 30.00\% & 39.83\%\\
OpenVLA + TT-VLA
& 34.58\% & 42.08\% & 42.92\%&39.83\%
& 57.08\% & 47.08\% & 42.92\% & 31.25\% & 31.33\% &41.93\%\\
$\Delta$
& +2.92\% & +0.83\% & +6.67\%&+3.45\%
& +2.50\% & +1.67\% & +2.08\% & +2.92\% & +1.33\% &+2.10\%\\
\rowcolor{mygray}
Relative Gain& \textcolor{ForestGreen}{$\uparrow$ 9.21\%} & \textcolor{ForestGreen}{$\uparrow$ 2.02\%}
& \textcolor{ForestGreen}{$\uparrow$ 18.40\%} &\textcolor{ForestGreen}{$\uparrow$ 9.54\%}
& \textcolor{ForestGreen}{$\uparrow$ 4.58\%}
& \textcolor{ForestGreen}{$\uparrow$ 3.67\%} & \textcolor{ForestGreen}{$\uparrow$ 5.10\%}
& \textcolor{ForestGreen}{$\uparrow$ 10.29\%} & \textcolor{ForestGreen}{$\uparrow$ 4.43\%} &\textcolor{ForestGreen}{$\uparrow$ 5.27\%}\\
\hline
OpenVLA-RL \tiny~\cite{liu2025can}
&82.08\%  &81.25\%  &81.25\% & 81.53\%
& 87.08\% &85.00\%  &85.83\%  & 64.17\% & 69.17\% & 78.25\%\\
OpenVLA-RL + TT-VLA
& 84.17\% &82.08\%  & 86.25\% &84.17\%
& 90.00\% & 86.25\% & 85.83\% & 65.83\% & 72.08\% &80.00\%\\
$\Delta$
& +2.09\% &+0.83\%  & +5.00\%&+2.64\%
& +2.92\% & +1.25\% & +0.00\% & +1.66\% & +2.91\% &+1.75\%\\
\rowcolor{mygray}
Relative Gain&\textcolor{ForestGreen}{$\uparrow$ 2.54\%}   &\textcolor{ForestGreen}{$\uparrow$ 1.02\%}  & \textcolor{ForestGreen}{$\uparrow$ 6.15\%}&\textcolor{ForestGreen}{$\uparrow$ 3.24\%}
& \textcolor{ForestGreen}{$\uparrow$ 2.08\%} &\textcolor{ForestGreen}{$\uparrow$ 1.47\%}  &$\uparrow$ 0.00\% & \textcolor{ForestGreen}{$\uparrow$ 2.59\%} & \textcolor{ForestGreen}{$\uparrow$ 4.21\%} &\textcolor{ForestGreen}{$\uparrow$ 2.23\%}\\
\hline

TraceVLA \tiny~\cite{zheng2024tracevla}
&55.00\%  &18.75\%  &7.08\% &26.94\%
&71.67\%  &67.08\%  &67.08\%  &45.83\%  &45.83\%  &59.50\%\\
TraceVLA + TT-VLA
&57.92\%  &21.50\% & 7.50\%&28.97\%
&72.92\%  &67.50\%  &67.92\%  &46.25\%  &47.08\%  &60.33\%\\
$\Delta$
&+2.92 \% &+2.75\%  &+0.42\% &+2.03\%
&+1.25\%  &+0.42\%  &+0.84\%  &+0.42\%  &+1.25\%  &+0.84\%\\
\rowcolor{mygray}
Relative Gain&\textcolor{ForestGreen}{$\uparrow$5.31\%}  &\textcolor{ForestGreen}{$\uparrow$14.67\%}  &\textcolor{ForestGreen}{$\uparrow$5.93\%} &\textcolor{ForestGreen}{$\uparrow$7.53\%}
&\textcolor{ForestGreen}{$\uparrow$1.47\%}  &\textcolor{ForestGreen}{$\uparrow$0.63\%}  &\textcolor{ForestGreen}{$\uparrow$1.25\%}  &\textcolor{ForestGreen}{$\uparrow$0.92\%}  &\textcolor{ForestGreen}{$\uparrow$2.73\%}  &\textcolor{ForestGreen}{$\uparrow$1.41\%}\\
\hline
\end{tabularx}

\vspace{3pt}
\setlength{\tabcolsep}{1.5pt}
\begin{tabularx}{\textwidth}{l||YYYYYYY|Y}
\hline \thickhline
\rowcolor{mygray}
&\multicolumn{8}{c}{\textit{Semantics}}\\
\rowcolor{mygray}
\multirow{-2}{*}{Model}
& \textbf{M-Obj. OOD} & \textbf{Instruct} & \textbf{M Recep.} & \textbf{Recep.}
& \textbf{Dist Recep.}  & \textbf{Object} & \textbf{M-Obj. IND}&\textbf{Avg.}\\ 
\hline\hline
Nora \tiny~\cite{hung2025nora}
& 10.00\% & 39.85\% & 16.67\% & 28.33\%
& 26.25\% & 17.08\% & 27.08\%  & 23.57\%\\ 
Nora + TT-VLA
& 11.25\% & 42.50\% & 18.75\% & 30.83\%
& 30.00\% & 18.33\% & 27.08\%   &25.53\%\\ 
$\Delta$
& +1.25\% & +2.92\% & +2.08\% & +2.50\%
& +3.75\% & +1.25\% & +0.00\%   &+1.96\%\\ 
\rowcolor{mygray}
Relative Gain& \textcolor{ForestGreen}{$\uparrow$ 12.50\%} & \textcolor{ForestGreen}{$\uparrow$ 7.38\%}
& \textcolor{ForestGreen}{$\uparrow$ 12.48\%} & \textcolor{ForestGreen}{$\uparrow$ 8.82\%}
& \textcolor{ForestGreen}{$\uparrow$ 14.29\%} & \textcolor{ForestGreen}{$\uparrow$ 7.32\%}
& $\uparrow$ 0.00\% &\textcolor{ForestGreen}{$\uparrow$ 8.33\%} \\   
\hline

OpenVLA \tiny~\cite{kim24openvla}
& 28.75\% & 49.58\% & 20.42\% & 33.33\%
& 43.75\% & 45.00\% & 49.58\%  &38.63\%\\ 
OpenVLA + TT-VLA
& 32.05\% & 50.17\% & 29.58\% & 37.50\%
& 45.00\% & 46.25\% & 50.00\%  &41.51\%\\  
$\Delta$
& +3.30\% & +0.28\% & +9.17\% & +4.17\%
& +1.25\% & +1.25\% & +0.42\%  &+2.88\%\\ 
\rowcolor{mygray}& \textcolor{ForestGreen}{$\uparrow$ 11.48\%} & \textcolor{ForestGreen}{$\uparrow$ 1.18\%}
& \textcolor{ForestGreen}{$\uparrow$ 44.90\%} & \textcolor{ForestGreen}{$\uparrow$ 12.50\%}
& \textcolor{ForestGreen}{$\uparrow$ 2.86\%} & \textcolor{ForestGreen}{$\uparrow$ 2.78\%}
& \textcolor{ForestGreen}{$\uparrow$ 0.85\%} &\textcolor{ForestGreen}{$\uparrow$ 7.54\%}\\  
\hline
OpenVLA-RL \tiny~\cite{liu2025can}
& 62.50\% &86.67\%  & 60.00\% &79.58\% 
& 80.42\% & 77.50\% & 77.50\%  &74.88\%\\    
OpenVLA-RL + TT-VLA
& 65.00\% & 90.00\% & 61.25\% & 79.58\%
& 80.83\% &78.33\%  & 80.00\%   &76.43\%\\   
$\Delta$
& +2.50\% &+3.33\%  & +1.25\% & +0.00\%
& +0.41\% &+0.83\%  & +2.50\%   &+1.55\% \\   
\rowcolor{mygray}
Relative Gain& \textcolor{ForestGreen}{$\uparrow$ 4.00\%} & \textcolor{ForestGreen}{$\uparrow$ 3.84\%}  & \textcolor{ForestGreen}{$\uparrow$ 2.08\%}
& $\uparrow$ 0.00\% & \textcolor{ForestGreen}{$\uparrow$ 0.51\%}  &\textcolor{ForestGreen}{$\uparrow$ 1.07\%}  &\textcolor{ForestGreen}{$\uparrow$ 3.23\%}    &\textcolor{ForestGreen}{$\uparrow$ 2.06\%}\\ 
\hline

TraceVLA \tiny~\cite{zheng2024tracevla}
&22.50\%  &59.17\%  &27.92\%  &47.50\% 
&55.83\%  &45.00\%  &57.92\%  &45.12\%   \\
TraceVLA + TT-VLA
&25.00\%  &60.00\%  &28.33\%  &51.25\% 
&55.83\%  &47.08\%  &60.00\%  &46.78\%   \\
$\Delta$
&+2.50\%  &+0.83\%  &+0.41\%  &+3.75\% 
&+0.00\%  &+2.08\%  &+2.08\%  &+1.66\%   \\
\rowcolor{mygray}
Relative Gain&\textcolor{ForestGreen}{$\uparrow$11.11\%}  &\textcolor{ForestGreen}{$\uparrow$1.40\%}  &\textcolor{ForestGreen}{$\uparrow$1.47\%} &\textcolor{ForestGreen}{$\uparrow$7.89\%}  
&$\uparrow$0.00\%  &\textcolor{ForestGreen}{$\uparrow$4.62\%}  &\textcolor{ForestGreen}{$\uparrow$3.59\%}  &\textcolor{ForestGreen}{$\uparrow$3.69\%}  \\
\hline
\end{tabularx}
\vspace{-6pt}
\caption{\textbf{Main results on unseen simulation tasks.} We report success rates (\%) across three generalization dimensions: \textit{Execution}, \textit{Vision}, and \textit{Semantics} on 4 state-of-the-art open-source VLAs (\ie, Nora, OpenVLA, OpenVLA-RL, and TraceVLA). $\Delta$ denotes the absolute improvement, and \textcolor{ForestGreen}{$\uparrow$} indicates relative gains. As shown in the table, across all baselines and task categories, TT-VLA consistently improves performance during test time, demonstrating substantial generalizability and constituting a novel angle for VLA adaptivity.}
\label{table:main_result}
\vspace{-10pt}
\end{table*}
\subsection{Experimental Setup}\label{subsec:exp-set}
\noindent\textbf{Environment and Task Settings.}
As stated in \S\ref{subsec:lr:ga}, 
TT-VLA aims to address the inherent vulnerability of current VLAs to unanticipated dynamics and domain shifts.
To evaluate this generalization capability on \textit{\textbf{unseen tasks}}, we test TT-VLA on both simulated and real-world tasks. 

For \textbf{\textit{simulation experiments}} (see \S\ref{subsec:sim-result}), 
we follow RL4VLA's~\cite{liu2025can} setup, focusing on a standard pick-and-place manipulation scenario. The agent receives an RGB observation and a natural-language instruction, and outputs a Cartesian end-effector delta along with a binary gripper command.
Specifically, as in~\cite{liu2025can}, we evaluate generalization along three dimensions: \textit{Execution}, \textit{Vision}, and \textit{Semantics}.
For \textit{Execution}, the initial poses of the robot, object, and receptacle are randomized, and an additional mid-episode object repositioning condition is introduced to evaluate robustness to dynamic variations during execution. 
For \textit{Vision}, both foreground and background appearances are varied through dynamic textures, unseen table surfaces, and image-level noise. 
For \textit{Semantics}, unseen objects, receptacles, and instruction paraphrases are introduced, along with multi-object, multi-receptacle, and distractor-receptacle tasks designed to assess compositional and linguistic generalization. Detailed task specifications are provided in Appendix~\S\ref{appendix:tasks_detail}.
All simulation experiments are conducted in ManiSkill 3 \cite{tao2024maniskill3} using a WidowX-250S robotic arm.

For \textbf{\textit{real-world evaluation}} (see \S\ref{subsec:real-world}), 
we study pick-and-place manipulation tasks on a Franka Research 3 platform. The agent similarly receives an RGB image and a task instruction, and outputs a Cartesian end-effector displacement together with a binary gripper command. We evaluate performance on nine unseen tasks designed to assess robustness to executional, visual, and semantic shifts.

\noindent\textbf{Implementation Details.}
For simulation, each task is executed for 80 trials across three random seeds using a $640\times 480$ RGB image as input.
For real-world experiments, each task is evaluated over 10 trials under consistent experimental conditions, including fixed camera viewpoints with an image resolution of $500\times 480$, controlled lighting, and static backgrounds. 
The policy is fine-tuned using LoRA~\cite{hu2022lora} with ranks $\{16, 32\}$. Learning rates are chosen from $\{1\times10^{-5}, 5\times10^{-5}, 1\times10^{-4}\}$, optimized with AdamW. A clip parameter $\epsilon$ of 0.2 is applied to enhance training stability. 
Each episode is executed with a fixed horizon of 160 steps.

\subsection{Simulation Results}\label{subsec:sim-result}
\noindent \textbf{Baselines.} We benchmark our proposed method against several state-of-the-art open-source VLA models, which span diverse architectural designs and training paradigms:
\begin{itemize}[noitemsep, left=1pt]
\item \textbf{Nora} \cite{hung2025nora} adopts Qwen-2.5-VL-3B~\cite{bai2025qwen2} as its backbone and employs the FAST+ tokenizer~\cite{pertsch2025fast} to enable efficient action sequence generation. Following~\cite{liu2025can} to ensure a strong initial policy, we further fine-tune Nora for 50k steps on a self-collected ManiSkill 3 dataset. 
\item \textbf{OpenVLA} \cite{kim24openvla} is one of the most widely used open-source VLA models. It is built on Llama-2-7B~\cite{touvron2023llama}. Consistent with~\cite{liu2025can}, we apply a warm-up fine-tuning phase of 10k steps prior to evaluation.
\item \textbf{OpenVLA-RL}~\cite{liu2025can} extends OpenVLA via reinforcement learning during training, enabling further task-specific policy refinement beyond supervised pre-training.
\item \textbf{TraceVLA}~\cite{zheng2024tracevla} is designed to enhance spatio-temporal reasoning through visual trace prompting. By encoding state–action histories as visual prompts, it better captures long-horizon dependencies and improves manipulation performance in interactive environments.
\end{itemize}
\noindent \textbf{Overall Performance.} 
As shown in Table~\ref{table:main_result}, our proposed method \textbf{consistently improves} the performance of \textbf{all} representative baseline models across a range of unseen tasks, demonstrating its broad applicability. For example, when applied to Nora, our method achieves gains on \textbf{14 out of 15} tasks, with relative improvements ranging from \textbf{5.26\% to 44.4\%}. The largest improvements are observed on Task Obj. Rep. ($ 44.4\%$) and Task Noise-s ($18.15\%$). Similarly, on OpenVLA, our method yields consistent performance improvements across nearly all tasks, including several large-margin gains (\eg, 44.9\% on Task M Recep. and 18.4\% on Obj. Rep.). 
Overall, while current methods generally focus on training-time sophisticated architectural improvements, we demonstrate that substantial generalizability across diverse baselines and unseen tasks can be achieved through a markedly more streamlined yet robust test-time adjustment. Given its capacity for dynamic adjustment based on unseen samples, our approach constitutes a significantly novel solution for VLA adaptivity~\cite{kachaev2025don, liu2025can}.

\begin{figure}[t!]
    \centering
    \includegraphics[width=0.47\textwidth,
    trim=0.7cm 0.5cm 1.8cm 0cm, clip]{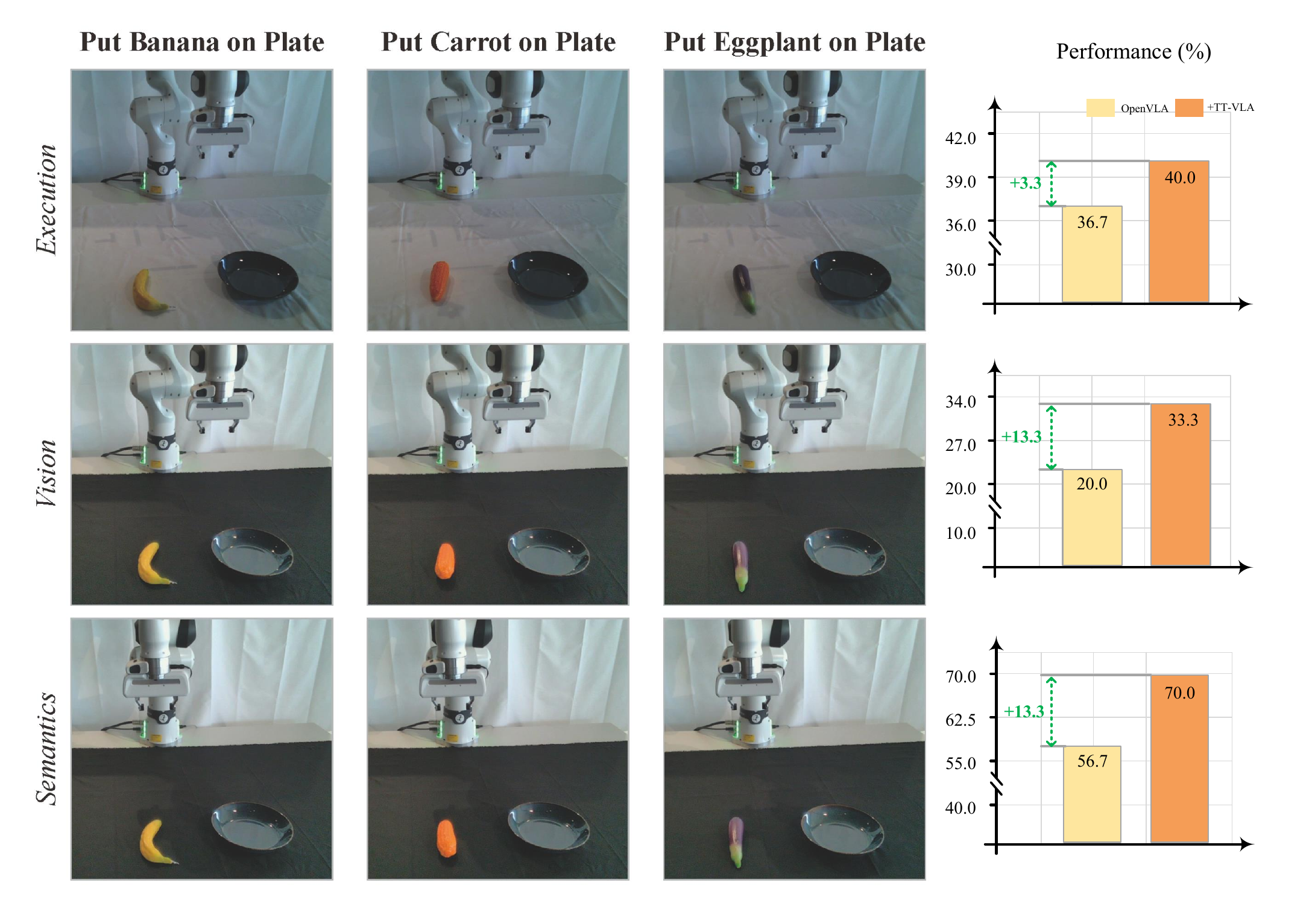}
    \vspace{-6pt}
    \caption{\textbf{Real-world setup and evaluation.} We evaluate nine real-world pick-and-place tasks covering \textit{Execution}, \textit{Vision}, and \textit{Semantics} generalization, with three tasks per category. The results show that TT-VLA consistently improves performance over baseline VLA models in real-world settings.}
    \vspace{-10pt}
    \label{fig:realtasks}
\end{figure}

\subsection{Real-World Results}\label{subsec:real-world}
We use OpenVLA as the base policy, and evaluate TT-VLA on nine unseen real-world tasks (see Fig.~\ref{fig:realtasks}). 
As seen, our method consistently improves performance over OpenVLA in real-world settings, demonstrating effective generalization beyond simulation and highlighting the practicality of test-time adaptation in real robotic deployments.

We further report
a representative case study ``put banana on plate'' in Fig.~\ref{fig:realtask_case_study}, where the robot grasps the banana and moves it toward the plate. During the original execution, the gripper temporarily deviates from 
the target region and moves away from the plate, signifying a substantial risk of task failure. 
However, by leveraging the dense, progress-based reward of TT-VLA, the policy enables rapid detection of task regression and on-the-fly behavioral adjustments.
The immediate reward feedback allows the VLA policy to correct deviations and realign with the task objective, ultimately completing the placement successfully. 
This example highlights the advantage of instantaneous, progress-aware rewards in enabling rapid recovery from execution errors. 
More real-world qualitative examples are shown in Appendix~\S\ref{appendix:moreevaluation}.

\begin{figure}[t!]
    \centering
    \vspace{-6pt}
    \includegraphics[width=0.47\textwidth, 
    trim=0.7cm 9.5cm 1.6cm 0cm, clip]{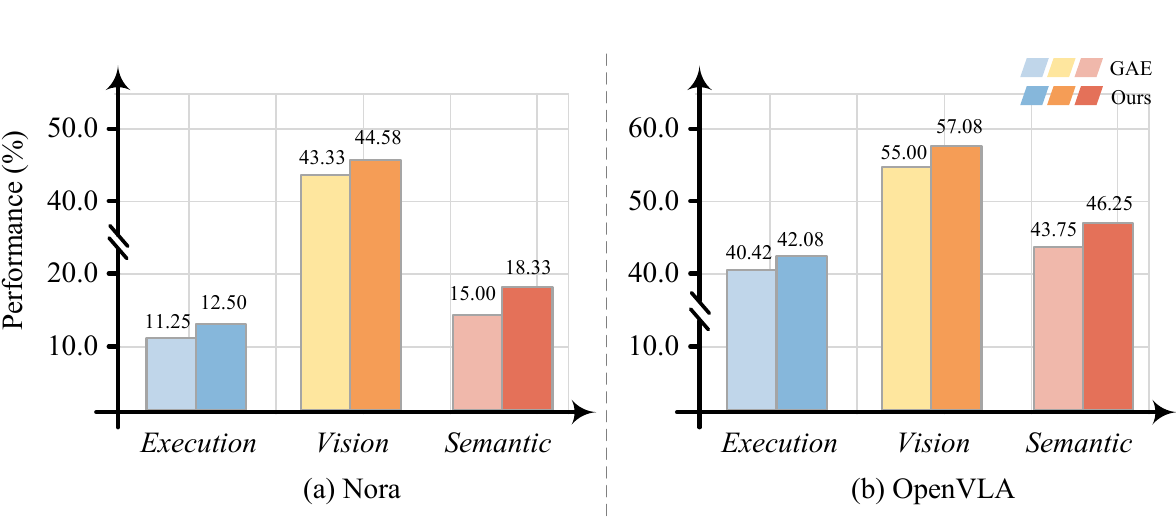}
    \vspace{-6pt}
    \caption{\textbf{Impact of reward design}. 
    The results show that our progress-based reward consistently outperforms the standard GAE across tasks and models.}
    \vspace{-10pt}
    \label{fig:ablationstudy}
\end{figure}

\begin{table}[b]
\centering
\vspace{-10pt}
\centering
\resizebox{0.45\textwidth}{!}{
\begin{tabular}{l||rr|rr|rr}
\hline \thickhline
\rowcolor{mygray}
Learning& \multicolumn{2}{c|}{\textit{Execution}} & \multicolumn{2}{c|}{\textit{Vision}} & \multicolumn{2}{c}{\textit{Semantics}}\\
\rowcolor{mygray}
Step&Nora&OpenVLA&Nora&OpenVLA&Nora&OpenVLA\\
\hline\hline
1
&10.83\%  &40.42\% &40.00\% &54.12\%  & 15.42\% &43.33\% \\
4
&11.25\%  &41.25\%  &42.08\% &56.25\%  & 17.50\% &\textbf{46.25\%} \\
8
&\textbf{12.50\%}  &\textbf{42.08\%}  &\textbf{44.58\%} &\textbf{57.08\%}  &  \textbf{18.33\%}&\textbf{46.25\%} \\
16
&11.25\%  &\textbf{42.08\%}  &43.33\% &55.42\%  & 17.50\% &45.42\% \\
\hline
\end{tabular}
}
\vspace{-6pt}
\caption{\textbf{Impact of learning step}. An update interval of 8 steps yield the best performance.}
\label{table:ablation_result}
\end{table}

\subsection{Diagnostic Experiments}\label{subsec:diag}
We conduct an ablation study on both Nora and OpenVLA using three representative unseen tasks. 

\noindent \textbf{Reward/Advantage Design.} We first analyze the effect of discounting in GAE (see Eq.~\ref{eq:GAE}). 
Specifically, we compare the standard GAE setting with nonzero discount factor and trace parameter (\eg, $\gamma > 0$, $\lambda > 0$) against our variant in which both $\gamma$ and $\lambda$ are set to zero. 
By eliminating discounting and trace accumulation, TT-VLA emphasizes how each individual action contributes to immediate progress rather than estimating long-term returns. 

Empirically, focusing on immediate progress yields consistent improvements in performance. For example, as shown in Fig.~\ref{fig:ablationstudy}, on the \textit{Vision} task with OpenVLA, our setting achieves a success rate of 57.08\%, compared to 55.00\% when using standard GAE. We attribute this performance gain to the fact that long-horizon returns can be unreliable in this setting, occasionally assigning overly optimistic advantage signals to ineffective actions. These results suggest that instantaneous feedback can be more effective than incorporating discounted future rewards during test time.

\noindent \textbf{Test-Time Training Steps.} We then explore the impact of model update frequency in TT-VLA by varying the update interval over different environment steps (\ie, 1, 4, 8, and 16). The number of learning steps is designed to balance the trade-off between rapid adaptation to newly collected data and the overall stability of the optimization process. Table~\ref{table:ablation_result} shows that updating the model every 8 steps yields the best performance. More frequent updates (\eg, 1 step) 
can destabilize training and limit the effectiveness of each update. In contrast, less frequent updates (\eg, 16 steps) delay policy improvement and reduce learning efficiency. These findings suggest that an intermediate update frequency achieves a balance between rapid policy adaptation and stable optimization. Additional details are provided in Appendix \S\ref{appendix:morediagnosticexp}.

\begin{figure}[t!]
    \centering
    \vspace{-5pt}
    \includegraphics[width=0.47\textwidth]{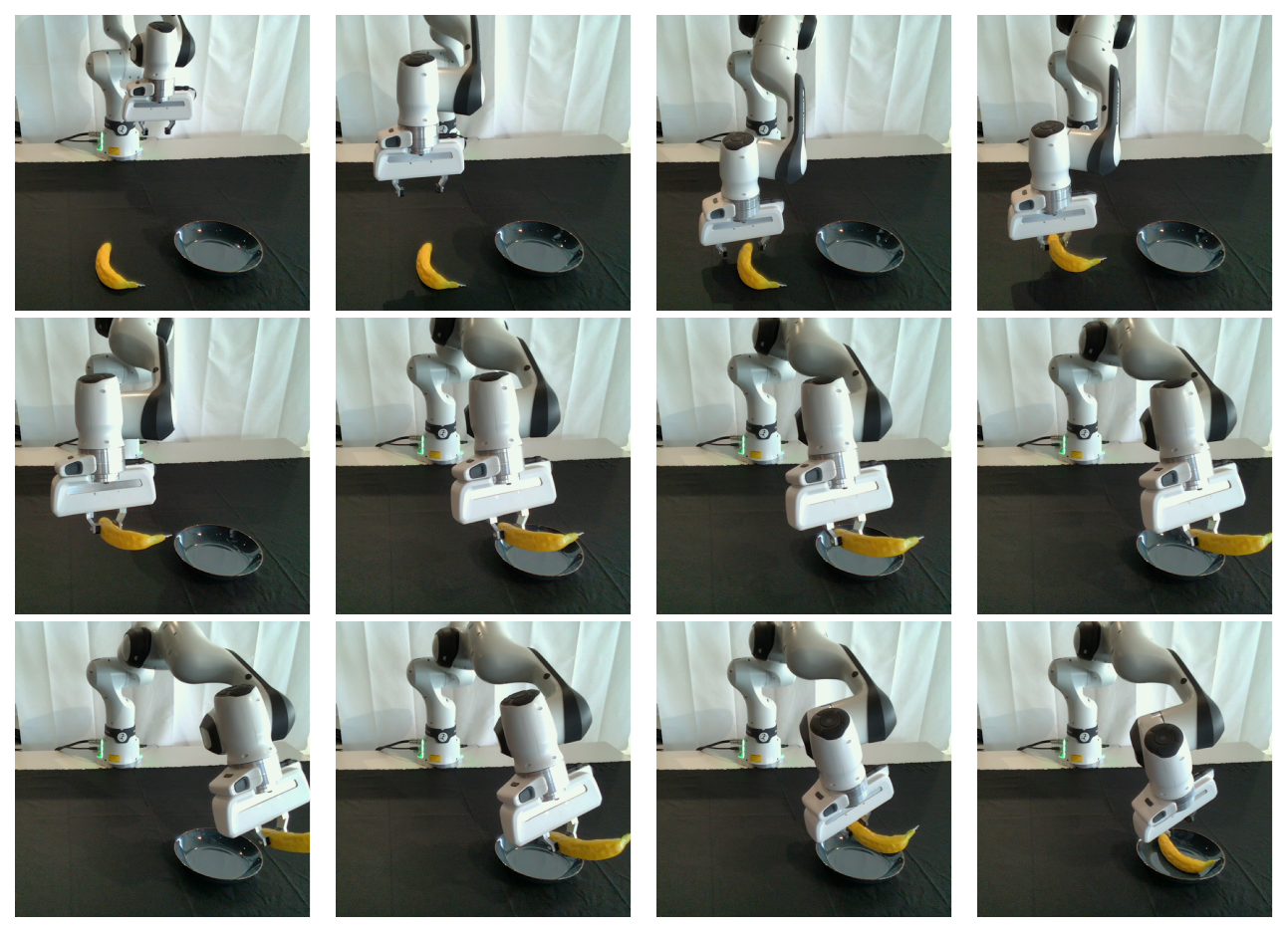}
    \vspace{-6pt}
    \caption{\textbf{Real-world case study} illustrates how TT-VLA’s instantaneous reward feedback enables rapid recovery from trajectory errors during deployment.}
    \vspace{-12pt}
    \label{fig:realtask_case_study}
\end{figure}

\subsection{Discussions on VLA Test-Time Training}\label{subsec:discussion}
As stated in \S\ref{appendix:subsec:TTL}, TTT was originally proposed for LLMs. 
A natural question is: \textit{Can test-time training methods in LLMs be directly applied to VLA models?} To address this question, we examine two representative approaches for VLAs: a self-supervised test-time training method, TLM~\cite{hu2025test}, and a test-time reinforcement learning method, TTRL~\cite{hu2025test}. Unless otherwise specified, we follow the same experimental setup as the diagnostic experiments, using the same tasks and baseline models for evaluation.

We first consider TLM~\cite{hu2025test} that enables test-time adaptation by directly minimizing input perplexity without any external supervision. 
When applied to VLAs, TLM updates model parameters by optimizing the likelihood of test-time observation sequences. 
As shown in Table~\ref{table:discussion_result}, this strategy results in lower performance gains than TT-VLA. 
The reason is that,
unlike pure language tasks, VLA tasks involve complex interactions between perceptions, instruction understanding, and actions. Updates driven solely by input likelihood may overly emphasize representational consistency rather than task-oriented decision making. As a result, 
self-supervised test-time objectives cannot readily translate to the VLA domain.

We further compare TT-VLA with TTRL~\cite{hu2025test}, a recently proposed test-time reinforcement learning approach. TTRL performs test-time adaptation by sampling multiple candidate outputs for each input and constructing a consensus pseudo-label via majority voting~\cite{shao2024deepseekmath}. This pseudo-label is then used to construct rule-based rewards, where outputs that match/mismatch the pseudo-label receive positive/zero rewards.
As reported in Table~\ref{table:discussion_result}, TTRL underperforms our proposed TT-VLA, 
indicating that the consensus-based pseudo-label is less effective for VLA tasks. 
One possible reason is that
majority voting does not reflect action quality, and reward signals derived from output agreement fail to provide 
task-aligned learning signals, thereby limiting the effectiveness of VLA test-time updates. More details of TLM and TTRL are provided in Appendix \S\ref{appendix:moretttdiscussions}.

\begin{table}[t]
\centering
\vspace{-3pt}
\resizebox{0.45\textwidth}{!}{
\begin{tabular}{l||rr|rr|rr}
\hline \thickhline
\rowcolor{mygray}
& \multicolumn{2}{c|}{\textit{Execution}} & \multicolumn{2}{c|}{\textit{Vision}} & \multicolumn{2}{c}{\textit{Semantics}}\\
\rowcolor{mygray}
\multirow{-2}{*}{Model}
&Nora&OpenVLA&Nora&OpenVLA&Nora&OpenVLA\\
\hline\hline
TLM
&11.25\%  &40.42\%  &41.25\%  &52.50\%  &16.67\% &42.9\%\\
TTRL
&10.42\%  &40.83\%  &39.58\% &51.42\%  &16.25\%  &41.76\% \\
\rowcolor{mygray}
Ours
&\textbf{12.50\%}  &\textbf{42.08\%}  &\textbf{44.58\%} & \textbf{57.08\%} &\textbf{18.33\%}  &\textbf{46.25\%} \\
\hline
\end{tabular}}
\vspace{-4pt}
\caption{\textbf{Comparison of TTT methods}. We compare TT-VLA with TLM and TTRL (see \S\ref{subsec:discussion}). As seen, TT-VLA achieves superior performance, highlighting the importance of progress-based reward for effective test-test adaptation in VLAs.}
\label{table:discussion_result}
\vspace{-12pt}
\end{table}

%% file: sec/6_conclusion.tex
\section{Conclusion}
\label{sec:conclusion}
While VLA models have gained significant popularity on closed-form benchmarks, this work focuses on the flexibility of applying these models in evolving environments via test-time reinforcement learning.
Experimental results demonstrate that TT-VLA consistently enhances performance on unseen tasks across diverse simulated and real-world scenarios, as well as across various VLA backbones. 
We believe that our framework provides pioneering and foundational contributions to VLA models to flexibly refine action policies 
under dynamic, previously unseen test-time cases.

%% file: sec/acknowledge.tex
\section*{Acknowledgments}
CL, YL, QZ, WY, DL, CH are not supported by any funds in this work.
The views and conclusions contained herein are those of the authors and should not be interpreted as necessarily representing the official policies or endorsements, either expressed or implied, of the U.S. Naval Research Laboratory (NRL) or the U.S. Government.

%% file: sec/X_suppl.tex
\clearpage

\appendix
\renewcommand{\thesection}{S\arabic{section}}
\renewcommand{\thetable}{S\arabic{table}}
\renewcommand{\thefigure}{S\arabic{figure}}
\setcounter{table}{0}
\setcounter{figure}{0}
\setcounter{section}{0}
\centerline{\textbf{SUMMARY OF THE APPENDIX}} 
\vspace{0.5em}

This appendix contains additional experimental results and discussions of our work, organized as:
\begin{itemize}
\setlength{\itemsep}{0pt}
  \item \S\ref{appendix:MRL} provides \textbf{more related works} on VLA models
  
  \item \S\ref{appendix:proof_lemma} provides \textbf{Lemma proof}. 
  
  \item \S\ref{appendix:tasks_detail} includes \textbf{more details on tasks}.
  \item \S\ref{appendix:morediagnosticexp} supplies additional information on \textbf{diagnostic experiments}.
  \item \S\ref{appendix:moretttdiscussions} supplies additional discussions on \textbf{Test-Time Training}.
  \item \S\ref{appendix:moreevaluation} provides \textbf{more qualitative results}.
  \item \S\ref{appendix:grpo} adds discussions on the practicalness of using \textbf{Test-Time GRPO in VLAs}.
  \item \S\ref{appendix:sec:reproduce} includes the \textbf{reproducibility statement and pseudo code} of our method. 
  \item \S\ref{appendix:sec:technical_contributions} highlights the \textbf{technical contributions} of our method.
  \item \S\ref{appendix:License} offers a \textbf{summary of licenses and consent}, and lists usage terms for all models and datasets. 
  
  \item \S\ref{appendix:ethics} includes additional discussions on \textbf{ethics concerns}.
  \item \S\ref{appendix:Discussion} discusses on \textbf{future directions}, highlighting potential areas for further research.
  \item \S\ref{appendix:disclosure} provides an \textbf{AI disclosure}, and notes that AI assistance was limited to grammar checking. 
\end{itemize}

\section{More Related Works}
\label{appendix:MRL}

\subsection{More Discussions on VLA}
\label{appendix:subsec:VLA}

Recent studies~\cite{Brohan2023rt1, MeesHB2022What, Pong2020Skew} have advanced the potential of large-scale Vision Language Models (VLMs) as key enablers for generalist robots, demonstrating promising generalization across a variety of scenes~\cite{Brianna2023RT2, Jiang2022VIMA, Octo2024Octo, huang2023embodied, li2023vision, cui2024collaborative, wang2024large}. They generally achieve action planning via two main branches: I. Discretization-based approaches~\cite{kim24openvla, Brohan2023rt1, Brianna2023RT2}, which discretize the action space into a small set of action tokens; and II. Diffusion-based approaches~\cite{Cheng2023DP1, xian2023chaineddiffuser, Michael2023DP2, liang2023adaptdiffuser, ajayconditional}, which integrate diffusion heads for action prediction.

In our study, we focus on and generalize discretization-based approaches. The reason is that most diffusion-head VLA models adopt a separate action decoder, typically a latent diffusion process that maps visual and instruction embeddings to an action embedding stream, followed by a shallow MLP to regress the robot’s joint space~\cite{wen2025diffusionvla}. This design renders reinforcement-learning (RL) optimization impractical (\ie, also for diffusion large language model (DLLM) RL optimization~\cite{wang2025revolutionizing}) for three technical reasons: 
(i) the resulting policy is implicit and does not expose a tractable per-step log-likelihood ($\log \pi_\theta(a\mid s)$), precluding policy-gradient methods (\eg, REINFORCE~\cite{sutton1999policy}/PPO~\cite{schulman2017proximal}) and entropy regularization; 
(ii) action emission requires tens of denoising iterations per control step, creating an inner stochastic chain misaligned with environment time, which breaks step-wise credit assignment; and 
(iii) the diffusion noise-prediction objective is mismatched with return-based RL objectives, while the terminal MLP head is effectively deterministic, suppressing exploration. 
However, we notice a very recent paper dVLA~\cite{wen2025dvla} decodes actions as a discrete, autoregressive token sequence conditioned on state/instruction, making the current RL attempts possible to apply to diffusion-based approaches. While the code is not publicly available for us to evaluate, we highlight that our method can be naturally applied to this line of research.

\subsection{RL Methods for VLA}
\label{appendix:subsec:vla-policy}

As we mentioned in our study, recently, some efforts have attempted to apply RL to the VLA training stage, leaving the test-time adjustments underexplored. In light of this view, we aim to fill the last puzzle of on-the-fly policy adaptation. We list some research with high impact on the integration of RL on VLAs. 

GRAPE~\cite{zhang2024grape} uses Direct Preference Optimization (DPO)~\cite{rafailov2023direct} to train VLAs by integrating human preferences. ConRFT~\cite{chen2025conrft} introduces Reinforced Fine-Tuning~\cite{trung2024reft} to train VLAs in real-world environments, iteratively training VLAs
through alternating RL and SFT rounds. ReinboT~\cite{zhang2025reinbot} rises dense reward design and optimized VLA through reward maximization. iRe-VLA~\cite{guo2025improving} proposed an iterative training framework that combines SFT and RL stages to address training instability and computational overhead. 
RIPT-VLA~\cite{tan2025interactive} employs REINFORCE Leave-One-Out (RLOO)~\cite{ahmadian2024back} for VLA training. \cite{liu2025can} investigates RL’s impact on VLA generalization capabilities, demonstrating improvements
over SFT in unseen environments, objects, and textures. 
VLA-RL~\cite{lu2025vla} applies the PPO; TGRPO and SimpleVLA-RL~\cite{chen2025tgrpo, li2025simplevla} evaluate trajectories and optimize VLA with GRPO variants; RFTF~\cite{shu2025rftf} uses value models to generate dense rewards in embodied scenarios for VLA online RL; and SRPO~\cite{fei2025srpo} leverages a world model to generate progress-based dense rewards.
Though promising, it is important to note that current RL-based approaches all operate during training, while real-world deployments inevitably involve evolving conditions and distributional shifts at test time, necessitating VLAs capable of adaptive adjustment in response.
The approach most relevant to our work is EVOLVE-VLA \cite{bai2025evolve}, which utilizes task progress as a reward signal for reinforcement learning. However, we should notice that EVOLVE-VLA optimizes the policy using GRPO, which incurs substantial computational overhead and is therefore less suitable for real-time robotic deployment. This limitation becomes particularly pronounced in real-world robotic settings, where strict latency constraints are critical. 

\section{Lemma Proof}
\label{appendix:proof_lemma}
In this section, we provide the proof of Lemma~\ref{lemma:1}, which characterizes the relationship between GAE and the reward-only advantage used in TT-VLA. This result formally justifies our value-free test-time optimization objective.
\begin{proof}
(i) For (a): When \(\lambda=0\), the geometric weighting term \((\gamma\lambda)^l\) 
vanishes for all \(l>0\), Utilizing (14), it yields \(A_t=\delta_t\).\\
\noindent
(ii) For (b): When \(\gamma=0\), (14) and (15) respectively yields
\begin{equation}
\vspace{-0.2em}
\scalebox{0.85}{$
A_t=\delta_t,\; \delta_t = r_t - V(s_t), 
$}
\end{equation}
(16) implies that when $V(s)\equiv 0$, there holds
\begin{equation}
\vspace{-0.2em}
\scalebox{0.85}{$
A_t =  \delta_t = r_t , 
$}
\end{equation}
\noindent
which completes the proof.
\end{proof}

\section{Task Details}
\label{appendix:tasks_detail}
For simulation tasks, we follow \cite{liu2025can} to define three dimensions of generalization problems for unseen tasks, which are \textit{Execution}, \textit{Vision}, and \textit{Semantics}.

The training task setting: At the beginning of each episode, an object is sampled from the 16 training objects and a table appearance is sampled from the 16 training textures. The object and the receptacle (yellow plate) are placed on the table, with their positions uniformly randomized within a rectangular region. The language instruction follows the template “put $O$ on $R$”, where $O$ and $R$ denote the object and receptacle names, respectively.

\noindent \textbf{\textit{Execution}} explores changes in the initial positions of both the object, the receptacle, the robot initial pose, and mid-episode changes in the object’s position during task execution. 
\begin{itemize}
\item Unseen Object \& Receptacle Position (Obj. Pos.): The object and the receptacle are placed on the table, with their positions randomized within a larger square region that surrounds the original rectangular area. All other settings follow the Training setting.
\item Unseen Robot Init Pose (Robot Pose): At the start of each episode, the initial poses of all robots are randomized instead of being fixed as in the Training setting. All other settings remain identical to the Training setting.
\item Mid-Episode Object Reposition (Obj. Rep.): At the fifth timestep of each episode, the object is teleported to a new randomly sampled position on the table. All other settings remain identical to the Training setting.
\end{itemize}
\noindent \textbf{\textit{Vision}} includes both foreground and background changes, as well as image-level Dynamic Noise, applied with either weak or strong intensity.
\begin{itemize}
\item Unseen Table (Table): The table appearance is sampled from 5 unseen appearance.
\item Weak Dynamic Texture (Texture-w): In addition to sampling an object and a table appearance, a texture is selected from the 16 available textures at the start of each episode. This texture is cropped and resized at each timestep differently, and overlaid onto the object, receptacle, and robot arm with a transparency factor of 0.3.
\item Strong Dynamic Texture (Texture-s): The settings matches the Weak Dynamic Texture setting, except that the image mixing transparency is increased to 0.5.
\item Weak Dynamic Noise (Noise-w): In addition to sampling an object and a table appearance, a texture is selected from the 16 available textures at the start of each episode. The texture is cropped and resized at each timestep differently and overlaid over the entire image with a transparency factor of 0.3.
\item Strong Dynamic Noise (Noise-s): The settings matches the Weak Dynamic Noise setting, except that the image mixing transparency is increased to 0.5
\end{itemize}

\noindent \textbf{\textit{Semantics}} considers previously unseen variations in Objects, Receptacles, and Instruction Phrasings. 
\begin{itemize}
\item Unseen Objects (Object): The object is sampled from 9 unseen objects, while all other settings follow the Training setting.
\item Unseen Receptacles (Recep.): In addition to sampling an object and a table appearance, a receptacle is selected from 16 unseen receptacles at the start of each episode, replacing the default training receptacle (yellow plate). All other settings follow the Training setting
\item Unseen Instruction Phrasing (Instruct): In addition to sampling an object and a table appearance, a language instruction template is selected from 16 unseen templates (Same as \cite{liu2025can}) at the start of each episode, replacing the default instruction (“put $O$ on $R$”). All other settings follow the Training setting. 
\item Seen Multi-Object (M-Obj. (IND)): At the beginning of each episode, two distinct objects are sampled from the 16 training objects along with one of the 16 training table appearances. Both objects and the receptacle (yellow plate) are placed on the table, with their positions randomly initialized within a rectangular region.
\item Unseen Multi-Object (M-Obj. (OOD)): Two distinct objects are sampled from the nine unseen objects, with all other settings identical to the  Seen Multi-Object settings.
\item Distractive Receptacle (Dist Recep.): In addition to sampling an object and a table appearance, a distractor receptacle is selected from 16 unseen receptacles at the start of each episode and placed on the table without being used in the task. All other settings follow the Training setting.
\item Multi-Receptacle (M Recep.): At the beginning of each episode, an object is sampled from the 16 training objects, two distinct receptacles are sampled from the 16 unseen receptacles, and a table appearance is selected from the 16 training textures. The object and both receptacles are placed on the table, with their positions randomly initialized within a rectangular region.
\end{itemize}

For real-world evaluation, we assess our method on nine unseen manipulation tasks designed to test generalization across execution, vision, and semantic dimensions. The execution tasks consist of are “put banana on plate”, “put lemon on plate”, “put apple on plate” under different initial robot configurations, evaluating robustness to variations in starting states. The vision tasks use the same instructions but introduce different background appearances to assess visual generalization. The semantic tasks also follow the same instruction templates but involve an unseen plate at test time, evaluating the model’s ability to generalize to novel semantic contexts. The nine tasks are illustrated in Fig.~\ref{fig:realtasks}.

\section{Additional Details on Diagnostic Experiments }
\label{appendix:morediagnosticexp}
This section provides additional implementation details for the diagnostic experiments discussed in \S\ref{subsec:diag}. We conduct diagnostic experiments using Nora and OpenVLA. We evaluate one task from each dimension: \textit{execution}, \textit{vision}, and \textit{semantics}. Specifically, we use Task Robot Pose (\textit{execution}), Task Table (\textit{vision}), and Task Object (\textit{semantics}) for all ablations to ensure controlled and comparable evaluations across settings. 

For Advantage Design, the standard GAE baseline uses a discount factor of $\gamma = 0.99$ and a trace parameter of $\lambda = 0.95$, with a truncated horizon length of $l = 8$ for advantage estimation. In contrast, our method disables both discounting and trace accumulation by setting $\gamma = 0$ and $\lambda = 0$, yielding a one-step, reward-only advantage.

\section{Additional Details on Test-Time Training Discussions}
\label{appendix:moretttdiscussions}
This section provides implementation details for adapting TLM and TTRL to VLA models.

For TLM, we follow the original formulation and perform test-time adaptation by minimizing the perplexity of the instruction prompt. Concretely, given a task instruction, we optimize the model parameters to reduce the negative log-likelihood of the instruction tokens, without relying on external supervision or environment rewards. We set the loss weighting coefficient to $\lambda = 0.1$ and use a threshold value of $0$ for triggering updates. The policy is updated every 8 environment steps. We apply LoRA to update the policy, using a rank of 32 and a learning rate of $1\times10^{-4}$.

For TTRL, we adapt the consensus-based test-time reinforcement learning framework to the VLA setting. At each decision step, we sample multiple candidate action tokens from the model to construct a pseudo-label via majority voting. We set the voting group size to 8. The reward function is defined as a binary signal: a reward of 1 is assigned if the sampled action token matches the pseudo-label, and 0 otherwise. Policy updates are performed at every environment step to accommodate the step-wise nature of action execution in real-time settings. We employ LoRA to update the policy parameters, with a rank of 32 and a learning rate of $1\times10^{-4}$.

\section{Additional Real-world Qualitative Results}
\label{appendix:moreevaluation}
This section presents additional qualitative results from real-world scenarios, complementing results in \S\ref{subsec:real-world} and further demonstrating the effectiveness of TT-VLA. Fig.~\ref{fig:morequalitive} presents three real-world rollouts of the ``put banana on plate'' task using TT-VLA. In the first episode, the robot initially grasps the banana but places it at an incorrect location. It then re-grasps the banana and successfully places it onto the plate. In the second episode, the robot grasps the banana and moves it to a position behind the plate; the policy subsequently corrects its direction and completes the placement. Similarly, in the third episode, the robot initially moves past to the right side of the plate before adjusting its motion to place the banana correctly. These qualitative results demonstrate TT-VLA’s ability to recover from execution errors and handle real-world uncertainties without retraining or human intervention.

\begin{figure*}[t]
    \centering
    \includegraphics[width=0.98\textwidth]{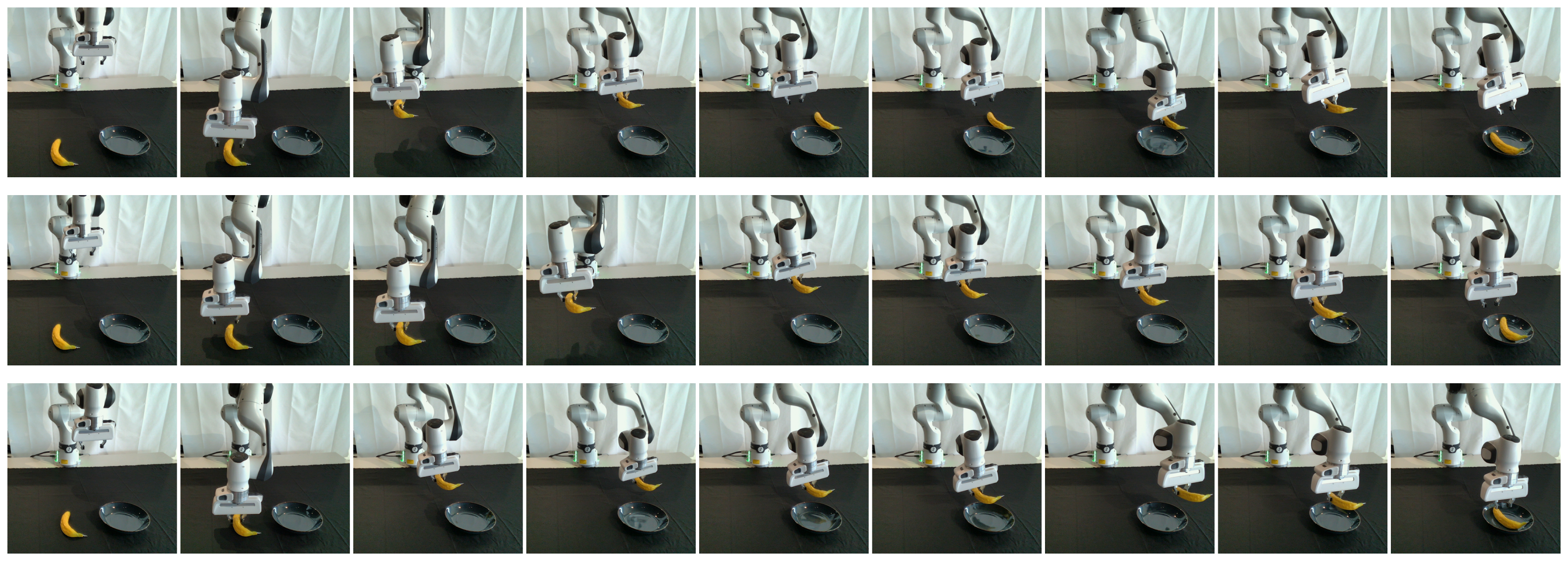}
    \vspace{-6pt}
    \caption{
\textbf{Additional real-world qualitative results.} Each row shows a real-world episode of the “put banana on plate” task, illustrating how TT-VLA adapts online to execution deviations and successfully completes the task using progress-based reward feedback.
}\vspace{-10pt}
    \label{fig:morequalitive}
\end{figure*}

\begin{algorithm}[t]
\caption{\textbf{TT-VLA  Pipeline}}
\label{alg:ttvla}
\KwIn{Pretrained VLA policy $\pi_\theta$, frozen progress estimator $\Phi(o_{0:t}, l)$, 
language instruction $l$, observation horizon $H$, update interval $K$, 
clipping threshold $\varepsilon$, learning rate $\eta$}
\KwOut{Task actions}

1: \hspace{1pt}\textbf{for} each episode \textbf{do}\\
2: \hspace{1pt}\hspace{8pt}Load pretrained VLA policy $\pi_\theta$; progress $p_{\text{0}} \leftarrow 0$; buffer $\mathcal{B} \leftarrow \emptyset$; Environment resets to initial state $s_0$\\
3: \hspace{1pt}\hspace{8pt}\textbf{for} each time step $t = 0, 1, 2, \dots ,T$ \textbf{do}\\
4:\hspace{8pt}\hspace{8pt}Sample $a_t \sim \pi_\theta(a_t \mid o_{t-1}, l)$, get $\log \pi_{\theta_{\text{old}}}(a_t \mid o_t)$, and execute $a_t$\\
5: \hspace{8pt}\hspace{8pt}Get new oberservation $o_{t+1}$\\
6: \hspace{8pt}\hspace{8pt}Compute $p_t \leftarrow \Phi(o_{0:t}, l)$ \hfill $\triangleright$ Eq.~\ref{eq:progress_estimation}\\
7: \hspace{8pt}\hspace{8pt}Compute $r_t \leftarrow p_t - p_{t-1}$ \hfill $\triangleright$ Eq.~\ref{eq:reward}\\
8: \hspace{8pt}\hspace{8pt}Store $(o_{t+1}, a_t, r_t, \log \pi_{\theta_{\text{old}}}(a_t \mid o_t))$ in $\mathcal{B}$\\
9: \hspace{8pt}\hspace{8pt}\textbf{if} $t \bmod K = 0$ \textbf{then}\\
10:\hspace{8pt}\hspace{8pt}\hspace{8pt}\textbf{for} each $(o_i, a_i, r_i, \log \pi_{\theta_{\text{old}}}(a_i \mid o_i)) \in \mathcal{B}$ \textbf{do}\\
11:\hspace{8pt}\hspace{8pt}\hspace{8pt}\hspace{8pt}Compute $r_i(\theta) \leftarrow \exp(\log \pi_\theta(a_i \mid o_i)$\\
\hspace{8pt}\hspace{8pt}\hspace{8pt}\hspace{8pt}\hspace{8pt}\hspace{8pt}$ - \log \pi_{\theta_{\text{old}}}(a_i \mid o_i))$    \\ 
12:\hspace{8pt}\hspace{8pt}\hspace{8pt}\hspace{8pt}Compute $L_i \leftarrow \min(r_i(\theta) \cdot r_i,\, $\\
\hspace{8pt}\hspace{8pt}\hspace{8pt}\hspace{8pt}\hspace{8pt}\hspace{8pt}$\text{clip}(r_i(\theta), 1-\varepsilon, 1+\varepsilon) \cdot r_i)$ \\\hfill $\triangleright$ Eq.~\ref{eq:PPO_clip}\\
13:\hspace{8pt}\hspace{8pt}\hspace{8pt}\hspace{8pt}Update policy parameters\\
\hspace{48pt}$\theta \leftarrow \theta + \eta \nabla_\theta \sum_i L_i$\\
14:\hspace{8pt}\hspace{8pt}Clear buffer $\mathcal{B}$\\
15:\hspace{8pt}\hspace{8pt}\textbf{end if}\\
16:\hspace{1pt} \textbf{end for}\\[4pt]
\end{algorithm}

\section{Discussions on Using Test-Time GRPO in VLAs}
\label{appendix:grpo}
In TT-VLA, we do not adopt Group Relative Policy Optimization (GRPO)~\cite{shao2024deepseekmath} due to \textit{two practical constraints} in test-time robotic deployment:
\begin{enumerate}
    \item GRPO relies on sampling multiple candidate trajectories or actions to update the policy, which introduces significant computational overhead and makes it inefficient for real-time test-time adaptation. Such sampling-based procedures are particularly unsuitable under test-time settings, where latency and responsiveness are critical.
    \item  In real-world robotic scenarios, sampled actions inevitably interact with the physical environment (\eg, touching or moving objects). It is thus infeasible to reset the environment to a previous state after each interaction. These constraints make GRPO-style sampling-based optimization impractical for test-time adaptation in physical environments. In fact, that is the practical reason that we redefine the advantage to depend only on the reward obtained from the current action (see \S\ref{eq:new_advantage}), as we want to prioritize rapid fitting of the current task rather than state accumulations.
\end{enumerate}

\section{Reproducibility}
\label{appendix:sec:reproduce}
TT-VLA is implemented in Pytorch~\cite{NEURIPS2019_9015}. Experiments are conducted on NVIDIA RTX 6000 Ada GPUs. 
To guarantee reproducibility, our full implementation shall be publicly released upon paper acceptance. 
We provide the pseudo code of TT-VLA in Algorithm~\ref{alg:ttvla}.

\section{Technical Contributions}
\label{appendix:sec:technical_contributions}

Our study presents three principal technical contributions: 
\begin{itemize}
    \item We introduce a test-time reinforcement learning framework for VLA models, enabling robots to adapt their policies on the fly during deployment without requiring retraining or environment resets. This capability directly addresses a key limitation of current VLA systems in real-world robotic settings, where conditions are dynamic and unpredictable.
    \item To cope with the severe data scarcity and latency constraints at inference time, we propose a dense, progress-based reward that provides stable and task-aligned learning signals at every step, allowing robots to refine their behavior during execution.
    \item Extensive experiments in both simulated and real-world robotic environments demonstrate that our approach consistently improves the robustness and success rates of existing SFT- and RL-based VLA models, highlighting its practical value for real-world robotic deployment.
\end{itemize}

\section{Asset License and Consent}
\label{appendix:License}
All models and datasets used in this work are publicly available.  
We strictly comply with their original licenses and use them only for non-commercial academic research.  
The contents of datasets do not represent our views or opinions.

\noindent\textbf{Models.}  
We utilize four open-source models:
Nora (MIT license),
OpenVLA (MIT license),
OpenVLA-RL (MIT license),
TraceVLA (MIT license).
All licenses permit academic research use; detailed terms are available via the original model repositories.  

\noindent\textbf{Datasets.}  
All simulation experiments were conducted in ManiSkill 3. The evaluated tasks are adopted from \cite{liu2025can}, and detailed task descriptions are provided in \S\ref{appendix:tasks_detail}.
The data (16400 demonstration trajectories) used to warm up the base models is collected following the same
procedure as in \cite{liu2025can}, and is generated automatically.

\noindent\textbf{Consent.}  
Our study does not involve crowdsourcing or human subjects. All results are derived from publicly available models and datasets.

\section{Ethics Concerns}
\label{appendix:ethics}
Test-time policy adaptation may increase the risk of unintended or unsafe behaviors, particularly in real-world robotic environments where erroneous actions can result in physical damage, equipment failure, or harm to surrounding objects and people. Because policy updates are performed online and are driven by interaction-derived feedback rather than explicit human supervision, unexpected environmental dynamics or imperfect reward signals may lead to behaviors that deviate from intended task objectives. To mitigate these risks, responsible deployment should incorporate safeguards such as constrained action spaces, explicit safety and termination constraints, and conservative update mechanisms. In addition, human oversight and monitoring remain essential, especially during deployment in safety-critical or unstructured environments, to ensure that adaptive behaviors remain aligned with task goals and safety requirements.

\section{Future Direction}\label{appendix:Discussion}
As discussed in \S\ref{appendix:subsec:VLA}, owing to the architectural distinctions between discretization-based and diffusion-based approaches, our study primarily focuses on the former. Future work should naturally extend our method to diffusion-based formulations, as TT-VLA provides a generalizable solution. 
Another promising direction is to utilize test-time adaptation (TTA) methods for effectively augmenting multimodal information.

It should be noted that these discussions on future direction present engineering
opportunities rather than insurmountable barriers.

\section{AI Disclosure}
\label{appendix:disclosure}
We acknowledge the use of GPT-5 for grammar checking only. The model was employed to correct grammatical errors while ensuring the original meaning and intent of the text remained unchanged.